%% file: main.tex
\documentclass[twoside]{article}

\usepackage[accepted]{aistats2021}

\usepackage[utf8]{inputenc}
\usepackage[english]{babel} 
\usepackage{amsfonts,amsthm, amssymb} 
\usepackage{float}
\usepackage{caption}
\usepackage{stmaryrd}
\usepackage{lipsum} 
\usepackage[ruled,vlined]{algorithm2e}
\usepackage{comment}
\usepackage[round]{natbib}   

\usepackage{enumitem}

\usepackage{hyperref}       
\hypersetup{ %
	pdftitle={Geometry-Aware Universal Mirror-Prox},
	pdfkeywords={},
	pdfborder=0 0 0,
	pdfpagemode=UseNone,
	colorlinks=true,
	linkcolor=blue, 
	citecolor=blue, 
	filecolor=blue, 
	urlcolor=blue, 
	pdfview=FitH,
	pdfauthor={Anonymous},
}

\usepackage{amsmath}               
  {
      \theoremstyle{plain}
      
  }
\usepackage{mathtools}

\input{def}

\begin{document}

%

%

\twocolumn[
\aistatstitle{Geometry-Aware Universal Mirror-Prox}
\aistatsauthor{ Reza Babanezhad \And Simon Lacoste-Julien}
\aistatsaddress{ SAIT AI Lab, Montreal \And  Mila, Universit\' e de Montr\' eal,
Canada CIFAR AI Chair }] 

\input{abstract}
\input{introduction}
\input{background}

\input{gemuni}

\input{bergsmooth}
\input{bregbnd}

\input{stoc}
\input{conc}
\section{Acknowledgements}
This research was partially supported by the Canada CIFAR AI Chair Program. Simon Lacoste-Julien is a CIFAR Associate Fellow in the Learning in Machines \& Brains program.


\bibliographystyle{plainnat}
\bibliography{ref}

\clearpage
\newpage
\appendix

\onecolumn
\input{supp}

\end{document}

%% file: def.tex
\newtheorem{thm}{Theorem}[section]
\newtheorem{cor}[thm]{Corollary}

\newtheorem{lem}[thm]{Lemma}

\theoremstyle{definition}
\newtheorem{defn}[thm]{Definition}

\theoremstyle{remark}

\newcommand{\sx}{\mathcal{X}}
\newcommand{\sk}{\mathcal{K}}
\newcommand{\sy}{\mathcal{Y}}
\newcommand{\sz}{\mathcal{Z}}

\newcommand{\sv}{\mathcal{V}}

\DeclareMathOperator*{\argmax}{arg\,max}
\DeclareMathOperator*{\argmin}{arg\,min}

\numberwithin{equation}{section}

\newcommand{\dpr}[2]{\left<#1,#2\right>}
\newcommand{\inv}[1]{#1^{-1}}
\newcommand{\dis}[2]{\delta(#1,#2)}

\newcommand{\Breg}[2]{\mathcal D_R(#1,#2)}
\newcommand{\x}{x^*}
\newcommand{\xt}{x_t}
\newcommand{\yt}{y_t}
\newcommand{\yn}{y_{t-1}}
\newcommand{\lr}{\eta} 
\newcommand{\zt}{Z_t} 

\newcommand{\tgt}{\tilde{g}_t}
\newcommand{\tmt}{\tilde{m}_t}


%% file: abstract.tex
\begin{abstract}
Mirror-prox (MP) is a well-known algorithm to solve variational inequality (VI) problems. VI with a monotone operator covers a large group of settings such as convex minimization, min-max or saddle point problems. To get a convergent algorithm, the step-size of the classic MP algorithm relies heavily on the problem dependent knowledge of the operator such as its smoothness parameter which is hard to estimate. Recently, a universal variant of MP for smooth/bounded operators has been introduced that depends only on the norm of updates in MP. In this work, we relax the dependence to evaluating the norm of updates to Bregman divergence between updates. This relaxation allows us to extends the analysis of universal MP to the settings where the operator is not smooth or bounded. Furthermore, we analyse the VI problem with a stochastic monotone operator in different settings and obtain an optimal rate up to a logarithmic factor.       
\end{abstract}

%% file: introduction.tex
\section{Introduction}
A large group of optimization problems can be formulated as a variational inequality (VI) problem~\citep{juditsky2016solving,nemirovski2004prox,juditsky2011solving}. These problems, including convex optimization and convex-concave saddle point problems, are ubiquitous in machine learning. For example, training a generative adversarial network (GAN)~\citep{goodfellow2014generative} model and its variants~\citep{arjovsky2017wasserstein,nowozin2016f} is framed as a zero-sum game instead of minimizing an empirical risk. Although the GAN's objective is not convex-concave w.r.t.\ its parameters in the traditional formulation~\citep{goodfellow2014generative}, however \citet{hsieh2019finding} proposed an alternative mixed Nash equilibrium formulation that is convex-concave. When there are more than two networks involved such as in~\citep{odena2017conditional}, the formulation goes beyond the min-max setting and moves to a smooth game formulation, and as pointed out by \citet{gidel2018variational}, is neatly unified by the VI framework and standard approaches from mathematical programming.

In VI~\citep{stampacchia1964formes}, we want to find $\x \in \sx$ such that for any other $x \in \sx$ the (VI) inequality 
\begin{equation}
    \label{eq:vi}
    \dpr{F(\x)}{x-\x} \geq 0 
\end{equation}
holds. In this equation, $F(x)$ is a monotone operator i.e.\ for any $x,y \in \sx$ we have $\dpr{x-y}{F(x)-F(y)} \geq 0$ and $\sx$ is a convex set. For convex minimization, $F$ is the subgradient of the objective function, while for the convex-concave saddle point problem $F$ is composed of (sub-)gradient and negative (sub-)gradient of the objective with respect to the primal and dual variables. However, the VI framework is more general than these two cases~\citep{nemirovski2004prox}.

Many algorithms have been designed to solve~\eqref{eq:vi}. One of the most commonly used is Forward-Backward (FB) splitting~\citep{bauschke2011convex}. However, this algorithm needs a cocoercivity assumption to guarantee the convergence~\citep{bauschke2011convex}. This assumption for convex minimization is equivalent to the \emph{Lipschitz smoothness} of the objective function~\citep{baillon1977quelques}. The extragradient algorithm (EG)~\citep{korpelevich1976extragradient} relaxes this assumption by requiring the underlying operator to be Lipschitz continuous. The mirror-prox (MP) algorithm~\citep{nemirovski2004prox,juditsky2011solving} generalizes the EG algorithm by incorporating the geometry of a given space by leveraging Bregman divergence. 

MP attains a $\mathcal O (1/T)$ ergodic convergence rate for Lipschitz continuous monotone operator and $\mathcal O (1/\sqrt T)$ for bounded operators. In terms of dependence to $T$, these rates are optimal~\citep{nemirovsky1992information,nemirovsky1983problem} i.e.\ this rate cannot be improved without further assumptions. To achieve these bounds, MP depends heavily on the properties of the problem at hand such as Lipschitz continuity parameter or an upper bound on the norm of the operator. However in practice, estimating these parameters is usually hard. Therefore for practical purposes, universal or adaptive algorithms that do not need the problem dependent information are advantageous~\citep{bach2019universal,rakhlin2013optimization,chiang2012online}.

Our contributions in this work are as follows:
\begin{itemize}
    \item We generalize the universal MP algorithm introduced by \citet{bach2019universal} by replacing the norm dependence in the step-size to a general Bregman divergence dependence. In this sense, the step-size becomes more geometry-aware. Moreover, for infinite dimensional spaces such as for the convex GAN formulation~\citep{hsieh2019finding} where computing the divergence is easier than the norm, this step-size is computationally advantageous. Similar to~\citep{bach2019universal}, we consider both smooth and non-smooth or bounded variational inequalities. Indeed developing a universal MP for this setting is one of the future works of \citet{antonakopoulos2019adaptive}. 
    \item Using a divergence instead of a norm allows us to extend our analysis for the cases where the variational inequality is not smooth (such as in robust svm). For this purpose, we borrow the notion of Bregman continuity from~\citep{antonakopoulos2019adaptive} and propose a novel step-size and prove the convergence of MP in this setting. 
    \item We extend the class of bounded operators by modifying the definition of relative continuity defined in~\citep{lu2019relative} to define \emph{Bregman boundedness}. We show that MP with the proposed universal and geometry-aware step-size converges with ergodic averaging.
    \item Being model agnostic allows us to analyze the stochastic or noisy variant of VI problem easily. For each setting we present the analysis of its stochastic version as well. In model agnostic approach unlike line search or parameter approximation based methods, the step-size is independent of operator evaluation that makes it suitable to extend its analysis for stochastic setting.     
\end{itemize}

\textbf{Related Work.} One of the first algorithms dealing with VI problem is the extragradient~\citep{korpelevich1976extragradient}. In each iteration, this algorithm makes two updates: to compute the \emph{extrapolated} point and updating the current point. The first update is look-ahead step to compute more stable direction. In the second update, it uses the gradient at the \emph{extrapolated} point to update the current point. 
Later \citet{korpelevich1983extrapolational} and \citet{noor2003new} analysed the asymptotic behaviour of extragradient algorithm for VI problem. Dual extragradient~\citep{nesterov2003dual} and MP~\citep{nemirovski2004prox} analysed VI non-asymptotically for smooth problems. Later, \citet{juditsky2011solving} analysed MP for bounded and stochastic setting. In all these settings, MP requires the prior knowledge about smoothness/boundedness of the problem. However, universal methods are oblivious to this knowledge. 

Universal algorithms have been proposed for different problems. \citet{yurtsever2015universal} and \citet{dvurechensky2018generalized} propose universal algorithms for smooth VI problem based on line search. However their methods are not model agnostic and is not appropriate for stochastic setting. Also they need an extra hyper-parameter as an accuracy tuning parameter. Similarly in the context of convex optimization, \citet{nesterov2015universal} introduces a universal method for smooth/bounded problem leveraging line search. However \citet{mcmahan2010adaptive}, \citet{duchi2011adaptive}, \citet{levy2017online}, and \citet{levy2018online} propose adaptive methods for convex optimization without using line search for smooth/bounded and noisy/noiseless settings.
 
Recently, the universal MP~\citep{bach2019universal} has been introduced and obtains optimal rate for smooth/bounded VI problem with/without noise. Their step-size is model agnostic and at each iteration depends on the norm of the past updates. In this work, we relax the step-size dependence from norm to Bregman divergence which makes the step-size more geometry-aware and for some setting computationally efficient. 

Using the Bregman divergence allows us to extends the analysis of universal MP to Bregman-smooth/Bregman-bounded settings. Bregman continuity (here we call it \emph{Bregman smoothness} for consistency) is introduced by \citet{antonakopoulos2019adaptive} to extend the analysis of the MP algorithm for non-smooth problems such as support vector machine, GAN with Kullback-Leibler losses or resource allocation problem. They show that under this new notion of smoothness, MP converges in deterministic and stochastic settings with problem dependent step-sizes. They also introduce an adaptive variant of MP that, similar to line-search based approach, at each iteration approximates the Bregman continuity parameter. Thus, it is not model agnostic and only converges in the deterministic setting. 

The notion of relative continuity~\citep{lu2019relative} was introduced to deal with non Lipschitz continuous objective functions such as the objective function of robust support vector machines or minimizing the maximum of convex quadratic functions. This notion defines an upper bound for the gradient of the loss function based on Bregman divergence between two given points. We adapt that definition for the VI problem by replacing the gradient with the operator value and call it \emph{Bregman bounded} operator. However, the analysis presented in~\citep{lu2019relative} is for mirror-descent algorithm~\citep{beck2003mirror}. Here we present the analysis of geometry-aware universal MP for \emph{Bregman bounded} operators. 

%% file: background.tex
\section{Background}
In this section, we present the general framework of variational inequality and the notion of Bregman divergence and gap function. We also present the MP algorithm.

\subsection{Preliminaries}
Let $\|.\|$ and $\|.\|_*$ represent a general norm and its dual norm respectively. Assume that $\dis{x}{y}$ represents a norm induced distance between $x$ and $y$ in $\sx$ where $\sx$ is a subset of a normed space $\sk$. A function $R:\sx \to \mathbb R$ is $\mu$-strongly convex if for all $x,y \in \sx$ 
\begin{equation*}
    \label{eq:sc}
    R(y) \geq R(x) + \dpr{\nabla R(x)}{y-x}+ \frac{\mu}{2} \dis{x}{y}^2. 
\end{equation*}
For example, if we set $\dis{x}{y} = \|x-y\|_2$, we recover the strong convexity definition in Euclidean space. A function $R$ is \emph{Lipschitz smooth} if for any $x$ and $y$, there exists a constant $L$ such that 
\begin{equation}
    \label{eq:lsm}
    \|\nabla R(x) - \nabla R(y) \|_* \leq L \|x-y\|. 
\end{equation}
If we replace $\nabla R$ in~\eqref{eq:lsm} with an operator $F$, we get the definition of a \emph{smooth operator}.

\textbf{Bregman Divergence.} Assume that $R$ is a differentiable and $\mu$-strongly convex function with respect to some distance function. The \emph{Bregman divergence} between $x,y \in \sx$ generated by the function $R$ is defined by 
\begin{equation*}
    \Breg{y}{x} = R(y) - R(x) - \dpr{\nabla R(x)}{y-x}.
\end{equation*}
Due to the strong convexity of $R$, we have 
\begin{equation}
    \Breg{y}{x} \geq \frac{\mu}{2} \dis{x}{y}^2. 
\end{equation}
If we set $R = 1/2\|.\|^2$ and $\dis{x}{y} = \|x-y\|_2$, we get $\Breg{y}{x} = 1/2\|x-y\|^2$ with $\mu =1$. As $R$ is chosen by the user, we assume without loss of generality for the rest of the paper that $\mu=1$.

\subsection{VI Framework}
Assume $\sk$ is a normed space with general norm $\|.\|$ and inner product $\left < .,.\right>$. Let $\sx$ be a convex subset of $\sk$ and $\sk^*$ represent the dual space of $\sk$. The variational inequality problem associated with the monotone operator $F:\sx \to \sk^*$ is defined as finding $\x$ such that 
\begin{equation}
    \label{eq:vi_cm}
    \dpr{F(\x)}{\x-x} \leq 0
\end{equation}
holds for all $x \in \sx$. If $F(x)$ is multi-valued, then the goal is finding $\x$ such that there exists a $g^* \in F(\x)$ that $\dpr{g^*}{\x-x} \leq 0$ holds~\citep{konnov2001combined}. For simplicity of presentation, we assume $F$ is a single-valued operator. 

In the following, we review some different formulations of VI for common problems in machine learning. 
\subsubsection{Convex Optimization.}
Let assume $f:\sx \to \mathbb{R}$ is a convex function. The convex optimization problem is finding $\x$ such that 
\begin{align}
    \label{eq:cvx_opt} 
    \x = \argmin_{x\in \sx} f(x). 
\end{align}
We assume that for any $x \in \sx$, we have access to the gradient of $f(x)$ when $f$ is smooth. To be covered by the VI framework, we set the monotone operator $F$ to be equal to the gradient of $f$, i.e.\ $F(x) = \nabla f(x)$. When $f$ is non-smooth with $\partial f(x)$ as its sub-differential set, then  $F(x)=\partial f(x)$ is a set-valued operator. If $f$ is \emph{Lipschitz smooth} with parameter $L$, then we say $F$ is \emph{Lipschitz smooth} with parameter $L$ and we have: 
\begin{equation}
    \label{eq:sm_cvx_op}
    \|F(x)-F(y)\|_* \leq L \|x-y\|. 
\end{equation}
If $f$ is Lipschitz continuous with parameter $G'$, i.e.\ $\|f(x)-f(y)\|_* \leq G' \|x-y\|$, then we say $F$ is \emph{Lipschitz bounded} with parameter $G'$ such that for all $x \in \sx$  
\begin{equation}
    \label{eq:bnd_cvx_op}
    \|F(x)\|_* \leq G'. 
\end{equation}
 
\subsubsection{Saddle Point Problem}
For this problem we assume that $f:\sz \times \sy \to \mathcal R$ and $f$ is convex w.r.t.\ $z \in \sz$ and concave w.r.t.\ $y \in \sy$. The goal is finding $\x = (z^*, y^*)$ such that 
\begin{equation} \label{eq:saddlePoint}
    f(z^*, y^*) = \min_{z \in \sz} \max_{y \in \sy} f(x,y).
\end{equation}
Let $\nabla_z f$ and $\nabla_y f$ denote (sub-)gradients of $f$ w.r.t.\ $z$ and $y$ respectively. To formulate it in the VI framework, the monotone operator is 
\begin{equation}
    F(x) = \begin{pmatrix}
\nabla_z f(z,y)\\ -\nabla_y f(z,y)
\end{pmatrix},
\end{equation}
where $x=(z,y) \in \sx = \sz \times \sy$. Similar to convex optimization setting, we can define the \emph{Lipschitz boundedness} and \emph{Lipschitz smoothness} for $F$. For more information about existence and computation of these parameters, one can look at~\citep{juditsky2011first}.  

\subsubsection{Multi-Player Game}
Continuous game with a finite number of players goes beyond the min-max or saddle point problem. Consider a game with $N$ players where each of them takes an action in a continuous and convex space $\sk_i$. The goal of each player is to minimize their own loss function. Formally. let $x=(x_i) \in \mathbb \prod_{i=1}^{N}\sk_i$ contain actions of all players. Each player $i$ tries to optimize its objective $f_i(x_i|x/x_i)$ where $x/x_i$ means the actions of all other players are fixed. A solution $\x$ is Nash equilibrium if for every player $i$ we have 
\begin{align}
 f_i(\x_i|\x/\x_i) \leq  f_i(x_i|\x/\x_i) 
\end{align}
for any $x_i \in \sk_i$. It means that when the other players are in their equilibrium, the best action for player $i$ is $\x_i$. If $\sk_i$ is convex and compact and $f_i$ is convex, then there exist a Nash Equilibrium for the game~\citep{debreu1952social}. Let $F(x) = (\nabla f_1(x_1|x/x_1), ...,\nabla f_N(x_N|x/x_N))$ where $\nabla f_i(x_i|x/x_i)$ is (sub-)gradient of $f_i$. Then it has been shown that a solution to VI problem with operator $F$ is a Nash equilibrium for the multiplayer game problem~\citep{balduzzi2018mechanics,hsieh2020explore}.    
\subsection{MP and Gap Function}
Algorithm~\ref{alg:mrr_prx} presents the general framework for the MP algorithm. At every iteration $t$, it uses the current value of $\yn$ and $F(\yn)$ to compute an extrapolated prediction $\xt$. The next iterate value $\yt$ is obtained by computing the operator $F(\xt)$ at the \emph{extrapolated} point $\xt$ and using the divergence to $\yn$. In classic MP, $\lr_t$ depends on the smoothness parameter or upper-bound of the norm of the operator. In universal MP~\citep{bach2019universal}, $\lr_t$ depends on $\yn,\xt,$ and $\yt$ and a mild dependence to the variation of $R$ over $\sx$ or so-called Bregman diameter of $\sx$. 
Note that the output of the MP algorithm uses the ergodic average of the \emph{extrapolated points} $\xt$'s.
\begin{algorithm}
\textbf{Input. }{$T: \text{\# Iterations}$}\\
 $y_0= \displaystyle\argmin_{x\in  \sx} R(x)$\\  
$x_0=y_0$\\
 \For{$t$ in $1...T$}{
    $m_t=F(\yn)$\\
    $\xt = \displaystyle\argmin_{x \in \mathcal X} \left\{\lr_t \dpr{m_t}{x} + \Breg{x}{\yn}\right\}$\\
    $g_t=F(\xt)$\\
     $\yt = \displaystyle\argmin_{x \in \mathcal X} \left\{\lr_t \dpr{g_t}{x} + \Breg{x}{\yn}\right\}$
 }
 $\textbf{Output}:\bar{x}_T = \frac{1}{T} \sum_{t=1}^T \xt$
 \caption{MP Algorithm}
 \label{alg:mrr_prx}
\end{algorithm}
In Algorithm~\ref{alg:mrr_prx}, if we set $R=1/2\|.\|^2$ we recover the extragradient algorithm for Euclidean space. 

\textbf{Gap Function.} We denote by $\sx^*$ the set of all possible solution to VI~\eqref{eq:vi_cm}. To characterize the convergence of the MP algorithm, we use  the notion of \emph{gap function} or \emph{merit function}~\citep{larsson1994class,zhu1998convergence,bach2019universal} as
\begin{align}
    &\text{Gap}(x) := \sup_{y \in \sx} \Delta(x,y)    
\end{align}
where $\Delta:\sx \times \sx \to \mathbb R$ is convex w.r.t.\ $x$ and for all $x,y$ we have 
\begin{equation}
\label{eq:delta_func}
    \Delta (x,y) \leq \dpr{F(x)}{x-y}.    
\end{equation}
 In convex optimization one can show that a meaningful gap function is
 \begin{equation}
    \text{Gap}(x) = f(x) - \displaystyle\min_{y \in \sx}f(x).     
 \end{equation}
 For a saddle point problem~\eqref{eq:saddlePoint}, we can take 
 \begin{equation}
    \text{Gap}(z,y) = \displaystyle\min_{u \in \sz}f(u,y) - \max_{v \in \sy}f(x,v)
 \end{equation}
 For a monotone operator, one can use the monotonicity property and define
 \begin{equation}
    \Delta(x,y) = \dpr{F(y)}{x-y}.
 \end{equation}
 It has been shown in~\citep{nesterov2007dual,antonakopoulos2019adaptive} that $Gap(x)=0$ if and only if $x \in \sx^*$. For the simplicity we only consider the convex minimization gap function in our analysis for deterministic case. Specifically we use regret analysis and show that the average regret grows sub-linearly. Here is the regret definition for $T$ iterations  
\begin{equation}
    \text{Regret} =\sum_{t=1}^T f(\xt) - f(\x) 
\end{equation}

%% file: gemuni.tex
\section{Geometry-Aware Universal MP}
The universal MP algorithm~\citep{bach2019universal} proposes an adaptive way to set the step-size at each iteration of MP which is model agnostic. Therefore, we do not require the problem dependent knowledge as well as line-search based methods to implement the algorithm. To be precise, universal MP requires to know the variation of divergence generating functions in the domain $\sx$ 
\begin{equation}
D = \displaystyle{\max_{x \in \sx} R(x) - \min_{x \in \sx} R(x)} .    
\end{equation}
Then it analyses the convergence of MP for both smooth and bounded operator. The step-size proposed by \citet{bach2019universal} is 
\begin{align}
    \label{eq:sz_ump}
    &\lr_t = \frac{D}{\sqrt{G_0^2 + \sum_{i=1}^{t-1}Z_i^2}}\\
    &\zt^2 := \frac{\|\xt-\yn\|^2 + \|\xt - \yt\|^2}{5\lr_t^2}
\end{align}
where $G_0$ is a constant. The proposed step-size depends on the norm of the update at each iteration. In this section, we consider the same assumptions as in~\citep{bach2019universal} but relax the dependence to the norms by using a general Bregman divergence. This modification first makes the step-size more geometry-aware and also allows us to extend the analysis of universal mirror-prox to the settings where these assumptions do not hold. 

Here are the list of assumptions we consider for smooth and bounded settings:
\begin{description}
    \item[(A1)]\label{asp:a1} For any $x \in \sx$ we have $\Breg{\x}{x} \leq D^2$ and $\x \in \sx^*$ is a solution.
    \item[(A2)]\label{asp:a2} $F$ is $L$-smooth.  
    \item[(A3)]\label{asp:a3} For all $t$, $\|m_t\|_{*}  \leq G'$ and $\|g_t\|_{*}  \leq G'$.
    \item[(A4)]\label{asp:a4} For all $x \in \sx$, $\|F(x)\|_{*}  \leq G'$.
\end{description}
For this section we assume that $R$ is $1$-strongly convex w.r.t.\ $\|.\|$. 

\textbf{Smooth Setting.} The geometry-aware step-size for this setting is defined as 
\begin{align}
     &\lr_t = \frac{D}{\sqrt{G_0^2 + \sum_{i=1}^{t-1} Z_i^2}} \label{eq:ga_sz_smth}\\
     &Z_t^2 = \frac{\Breg{\xt}{\yn}}{2\lr_t^2}\label{eq:zt_smth}.
\end{align}
In~\eqref{eq:ga_sz_smth}, $G_0$ is an arbitrary constant. As is clear from~\eqref{eq:zt_smth}, $\lr_t$ depends on the Bregman divergence $\Breg{\xt}{\yn}$. In the appendix, we show that $\zt$ is bounded for all $t$ and in all settings. Moreover based on its definition, $\lr_t$ is non-increasing and $\lr_{t+1} \leq \lr_{t}$ and $\lr_t \leq \lr_1 = \frac{D}{|G_0|}$. The following theorem shows that for smooth operator, the regret is upper-bounded by a constant.

\begin{thm}
\label{thm:lsmooth}
Assume assumptions \textbf{(A1-3)} holds. Then if we set $ \lr_t$ as in~\eqref{eq:ga_sz_smth}, we have 
\begin{align}
    \label{eq:bnd_regret_lsmh}
     \text{Regret} \leq& C
\end{align}
where $C$ is a constant dependent on $L,G',G_0$ and $D$. 
\end{thm}
The proof of~Thm.~\ref{thm:lsmooth} can be found in Appendix~\ref{app:Lsmooth_GAUMP}. The result from~\eqref{eq:bnd_regret_lsmh} shows that 
\[
    \text{Regret}=\sum_{t=1}^T f(\xt) - f(\x)\leq C
\] 
for some constant $C$. If we divide both sides by $T$, thanks to the convexity of $f$, we get
\begin{equation}
    f(\bar{x}_T) - f(\x) \leq \mathcal O (\frac{1}{T}), 
\end{equation}
that shows sublinear convergence of $f(\bar{x}_T)$ toward $f(\x)$.

\textbf{Bounded Setting.} In this setting we assume that there is an upper bound for the norm of the monotone operator value at every $x \in \sx$. The geometry-aware step-size for this setting is defined as 
\begin{align}
&\lr_t= \frac{D}{\sqrt{G_0^2 + \sum_{i=1}^{t-1}Z_i^2}} \label{eq:ga_sz_bnd}\\
&Z_t^2 = \frac{\Breg{\xt}{\yn}+\Breg{\yt}{\xt}}{\lr_t^2}.\label{eq:zt_bnd}
\end{align} 
This step-size has the same properties as the smooth case one. 
\begin{thm}
\label{thm:lbounded}
Assume \textbf{(A1),(A4)} holds. If we set $\lr_t$ as in~\eqref{eq:ga_sz_bnd} we can bound the regret as follows
\begin{equation}
    \label{eq:bnd_regret_lbnd}
    \text{Regret} \leq O(\sqrt{T \log(T)} ).
\end{equation}  
\end{thm}
The rate we get here is optimal w.r.t.\ $T$ up to a $\sqrt{\log(T)}$ factor. The proof of~Thm.~\ref{thm:lbounded} can be found in Appendix~\ref{app:Lbounded_GAUMP}.

%% file: bergsmooth.tex
\section{Bregman Smoothness}
In many practical applications, the Lipschitz continuity of the operator fails to hold. This could be due to rapid growth of the objective function in its domain (such as in support vector machine model) or it shows singularity behaviour near to the border of the domain (such as in resource allocation problem~\citep{roughgarden2010algorithmic}). 

Bregman continuity has been introduced in~\citep{antonakopoulos2019adaptive} to solve this problem. Since the global norm is oblivious to the geometry of the domain space, they introduce the notion of local norm. Leveraging this notion of norm and adapting the definition of divergence generating function based on the local norm, they propose the Bregman continuity condition for non-smooth operators which for consistency we call \emph{Bregman smoothness}. In the following, we review these notions. 

\textbf{Local Norm~\citep{antonakopoulos2019adaptive}.} Let $\mathcal{Z} = span(\mathcal{X}-\mathcal{X})$ be a subspace of $\mathcal {V}$ which is spanned by all vectors $\{ x-x' | x,x' \in \sx\}$. Then a local norm on $\mathcal{X}$ is a continuous assignment of norm $\|.\|_{x}$ on $\sz$ at each $x \in \sx$. Respectively, the induced dual norm is defined as 
\begin{align}
    \label{eq:local_norm}
    \|v\|_{x,*} = \max_{z \in \mathcal{Z}} \{ |\dpr{v}{z}| : \|z\|_x \leq 1. \}
\end{align} 
Based on this definition, the divergence generating function $R$ is assumed to be strongly convex w.r.t.\ this norm i.e.\ $\dis{x}{y}=\|x-y\|_x$. Therefore for the Bregman divergence we have 
\begin{align}
    \Breg{y}{x} \geq \frac{1}{2}\|x-y\|^2_x
\end{align}
which adapts the lower bound of the divergence based on local norm. Finally we present the notion of \emph{Bregman smoothness} based on~\citep{antonakopoulos2019adaptive}. 
\begin{defn}
An operator $F: \sx \to \sk^*$ is \emph{$L_{\beta}$-Bregman smooth} if for all $x,y \in \sx$
\begin{align}
    \|F(y)-F(x)\|_{y,*} \leq L_{\beta} \sqrt{ 2 \Breg{y}{x}}.  
\end{align}
\end{defn}
We make the following assumption: 
\begin{description}
    \item[(B1)]\label{asp:b1} $F$ is \emph{$L_{\beta}$-Bregman smooth}. 
\end{description}
The following theorem shows that \textit{Regret} is upper bounded by a constant. Moreover the step-size up to a constant is the same as for the \emph{Lipschitz smooth} case.  
\begin{thm}
\label{thm:Bsmooth}
Assume assumptions \textbf{(A1),(A3)} and \textbf{(B1)} hold. Then if we set $ \lr_t$ as in~\eqref{eq:ga_sz_smth} we have 
\begin{equation}
    \label{eq:bnd_regret_bsmh}
     \text{Regret} \, \leq \, C
\end{equation}
where $C$ is a constant depending on $G',L_{\beta},G_0$ and $D$. 
\end{thm}
The proof of~Thm.~\ref{thm:Bsmooth} can be found in Appendix~\ref{app:Bsmooth_GAUMP}. 

\citet{antonakopoulos2019adaptive} need $\lr_t \leq \frac{\sqrt{K}}{L_{\beta}}$ to get convergence in deterministic MP where $K$ is strong convexity parameter of $R$. To be adaptive to $L_{\beta}$ they approximate $L_{\beta}$ in each iteration by the following equation: 
\begin{align}
    L^t_{\beta } = \frac{\|F(\xt) - F(\yn) \|}{\sqrt{\Breg{\xt}{\yn}}}, 
\end{align}
and $\lr_{t+1}$ is evaluated based on the following rule 
\begin{align}
    \lr_{t+1}=\begin{cases}
    \min\{\lr_t, \theta \sqrt{K}/L^t_{\beta }\} &\quad \text{if } \yn\neq \xt \\
    \lr_t &\quad \text{o.w.} 
    \end{cases}
\end{align}
where $\theta \in (0,1)$ is a hyper-parameter and it guarantees that $\lr_{t+1} \leq \lr_t$. However this approximation needs the exact evaluation of $F$ at every $x$. Therefore there is no analysis in~\citep{antonakopoulos2019adaptive} for a universal stochastic variant. Moreover, in their proof they need $\displaystyle\lim_{t\to \infty}\frac{\lr_{t}}{\lr_{t+1}} \to 1$ and also assume that this happens after time $t_0 \ll T$ which is a strong assumption to make. 

%% file: bregbnd.tex
\section{Bregman Boundedness}
The notion of relative continuity has been proposed in the convex optimization setting~\citep{lu2019relative} to deal with non-differentiable objective functions in mirror descent algorithm~\citep{nemirovsky1983problem,beck2003mirror}. This continuity is determined w.r.t.\ a function $R$ which is easy to compute. Here we present a modified version of relative continuity adapted for an operator and we call it \emph{Bregman boundedness}.
\begin{defn}
An operator $F$ is \emph{Bregman bounded} if there exists a constant $M$ such that for every $x,y \in \sx$ 
\begin{equation}
    \|F(x)\|_* \leq M\frac{\sqrt{\Breg{y}{x}}}{\|x-y\|}. 
\end{equation}
\end{defn}
If we set $R=1/2\|.\|^2$ then we recover the bounded setting where $\|F\|_* \leq M $. If  $R(x) = 1/3\|x\|^3_2$, one can show that~\citep{lu2019relative} 
\begin{align}
\|F(x)\|_* \leq M \sqrt{\|x\|_2 + 2 \|y\|_2}     
\end{align}
where $M=\frac{1}{\sqrt{3}}$. The upperbound here depends on both $x$ and $y$ which makes it relative and helps to deal with unbounded domain.   

We make the following assumption: 
\begin{description}
    \item[(C1)]\label{asp:c1} $F$ is $M$-\emph{Bregman bounded}. 
\end{description}
The following theorem shows that the same step-size as for the bounded case gives convergence for MP algorithm.
\begin{thm}
\label{thm:Bbounded}
Assume \textbf{(A1),(C1)} holds. If we set $\lr_t$ as in~\eqref{eq:ga_sz_bnd} we can bound the regret as follows
\begin{equation}
    \label{eq:bnd_regret_bbnd}
    Regret \leq \mathcal{O}(\sqrt{T\log(T)}).
\end{equation}  
\end{thm}
The proof of~Thm.~\ref{thm:Bbounded} can be found in Appendix~\ref{app:Bbounded_GAUMP}.

%% file: stoc.tex
\section{Stochastic Monotone Operator}
In this section, we present the analysis of stochastic variant of different settings. We assume that we have access to the noisy version of a monotone operator. Then we show that for each setting, using a geometry-aware step-size gives us optimal convergence rate up to a logarithmic factor. This is done without any prior knowledge about problem structure or noise. In this section, convergence analysis is based on the gap function value at $\bar{x}_T$ i.e.\ $Gap(\bar{x}_T) = \displaystyle\max_{x \in \sx} \Delta(\bar{x}_T,x)$. 

 Let denote the stochastic variant of $F$ by $\tilde{F}$. We assume there is an inexact oracle that every time we query for the operator value at $x \in \sx$, it returns $\tilde{F}(x)$. To adapt to stochasticity, we replace $g_t$ and $m_t$ with $\tgt$ and $\tmt$ in the Alg.~\ref{alg:mrr_prx}. We make the following assumptions
\begin{description}
    \item[D1] Unbiased estimator: $\mathbb E \left[\tilde{F}(x)|\mathcal H_t\right] = F(x)$
    \item[D2] Bounded variance: $\mathbb E \left[\|\tilde{F}(x)-F(x)\|^2|\mathcal H_t\right] \leq \sigma^2$
\end{description}
where $\mathcal H_t$ denotes the history (filtration) of the random variables up to time $t$.

The proof of the theorems in this section is mainly based on the following lemma that is presented in~\citep{bach2019universal}. 
\begin{lem}~\citep{bach2019universal}
\label{lem:mrtg_diff}
Let $\sk \in \mathbb R$ be a convex set and $R: \sk \to \mathbb R$ be $1$-strongly convex w.r.t.\ $\|.\|$. Also assume that for all $x \in \sk$ we have 
\begin{align*}
R(x) - \min_{y \in \sk} R(y) \leq 1/2 D^2.    
\end{align*}
 Then for any martingale difference sequence $(Z_i)_{i=1}^n \in \mathbb R^d$ and any random vector $X$ defined on $\sk$ we have 
\begin{equation}
    \mathbb E \left[ \dpr{\sum_{i=1}^n Z_i}{X}\right] \leq D/2 \sqrt{\sum_{i=1}^n \mathbb E \|Z_i\|_*^2},
\end{equation}
where $\|.\|_*$ is the dual norm of $\|.\|$. 
\end{lem}
To prove the above lemma, \citet{bach2019universal} assume without loss of generality that $R(0)=0$. So we also make this mild assumption in our proofs as well.  

In all of the stochastic settings we consider the following adaptive step-size with different constant $c$: 
\begin{align}
    &\lr_t= \frac{D}{\sqrt{G_0^2 + \sum_{i=1}^{t-1}Z_i^2}}\label{eq:stoc_sz}\\
    &Z_t^2 = \frac{\Breg{\xt}{\yn}+\Breg{\yt}{\xt}}{c^2\lr_t^2}
\end{align}

\subsection{Smooth Settings}
In this section we consider the Lipschitz and \emph{Bregman smooth} settings. The following theorem is for \emph{Lipschitz smooth} setting. 
\begin{thm}
\label{thm:stoc_lip_smt}
Assume \textbf{(A1-3)} and \textbf{(D1-2)}. If we set $\lr_t$ as in~\eqref{eq:stoc_sz} with $c=5$ we have 
\begin{align}
   \mathbb E \max_{x \in \sx }  \Delta(\bar{x}_T, x) \leq \mathcal O (\sqrt{\log(T)}/\sqrt{T}).
\end{align}
\end{thm}
The proof of~Thm.~\ref{thm:stoc_lip_smt} can be found in Appendix~\ref{app:Stoch_Lsmooth_GAUMP}. Similar to Thm~\ref{thm:stoc_lip_smt}, we have the same bound for the \emph{Bregman smooth} setting. 
\begin{cor}
\label{cor:stoc_Bsmooth}
Assume \textbf{(A1,A3,B1)} and \textbf{(D1-2)} hold. If we set $\lr_t$ as in~\eqref{eq:stoc_sz} with $c=5$ we have 
\begin{align}
    \mathbb E \max_{x \in \sx } \Delta(\bar{x}_T, x) \leq \mathcal O (\sqrt{\log(T)}/\sqrt{T}).
\end{align}
\end{cor}
Proof of above corollary is very similar to Thm.~\ref{thm:stoc_lip_smt}. It can be found in Appendix~\ref{app:Stoch_Bsmooth_GAUMP}. 

\subsection{Bounded Setting}
In this section we consider the Lipschitz and \emph{Bregman bounded} settings. For this section we consider the following assumptions 
\begin{description}
    \item[D3] $\|\tilde F\|_* \leq G'$
    \item[D4] $ \|\tilde g(x)\|_* \leq M \frac{\sqrt{\Breg{y}{x}}}{\|x-y\|}$
\end{description}
The following theorem shows the convergence for \emph{Lipschitz bounded} operator. 
\begin{thm}
\label{thm:stoch_Lbounded}
Assume \textbf{(A1),(D3)} holds. If we set $\lr_t$ as in~\eqref{eq:stoc_sz} with $c=1$ we have 
\begin{equation}
    \mathbb E \max_{x \in \sx } \Delta(\bar{x}_T, x) \leq \mathcal O (\sqrt{\log(T)}/\sqrt{T}).  
\end{equation}
\end{thm}
The proof of~Thm.~\ref{thm:stoch_Lbounded} can be found in Appendix~\ref{app:Stoch_Lbounded_GAUMP}.

The next theorem shows the convergence of stochastic MP under \emph{Bregman boundedness} condition. 
\begin{thm}
\label{thm:stoch_Bbounded}
Assume \textbf{(A1),(D4)} holds. If we set $\lr_t$ as in~\eqref{eq:stoc_sz} with $c=1$ we have 
\begin{equation}
    \mathbb E \max_{x \in \sx } \Delta(\bar{x}_T, x) \leq \mathcal O (\sqrt{\log(T)}/\sqrt{T}).  
\end{equation}
\end{thm}
The proof of~Thm.~\ref{thm:stoch_Bbounded} can be found in Appendix~\ref{app:Stoch_Bbounded_GAUMP}. In all theorem mentioned above, the $\mathcal O$ notation hides the dependence to $\sigma^2$ or noise level of the operator. The detail is presented in Appendix~\ref{app:Stoch_GAUMP}.  

%% file: conc.tex
\section{Conclusion}
 Universal algorithms are oblivious to the problem dependent information such as smoothness or continuity parameters. We consider the universal MP algorithm that merely depends on the variation of divergence generating function in the domain. We propose step-sizes which are more geometry-aware and depends on the Bregman divergence between updates in the MP algorithm. Using this new step-size allows us to extend the analysis of universal MP for the \emph{Bregman smooth/bounded} operators. Being model agnostic helps to easily extend the analysis of universal MP with geometry-aware step-size to the stochastic setting. Making the algorithm adaptive to the Bregman diameter is left for future research.  

%% file: supp.tex
\section{Preliminaries}
\label{app:preliminaries}
We recall the update for the mirror-prox algorithm. 
\begin{align}
    \label{app:alg:mp_x_update}
    & \xt = \argmin_{x \in \mathcal X} \lr_t \dpr{m_t}{x} + \Breg{x}{\yn}, \ \ \ m_t=F(\yn)\\
    \label{app:alg:mp_y_update}
    & \yt = \argmin_{x \in \mathcal X} \lr_t \dpr{g_t}{x} + \Breg{x}{\yn} , \ \ \ g_t=F(\xt) 
\end{align}
Here is a list of major assumptions we consider  

\begin{description}
    \item[(A1)]\label{app:asp:a1} for any $x \in \sx$ we have $\displaystyle\max_{x\in\sx} R(x) - \min_{x\in\sx} R(x) \leq D^2$ or $\Breg{\x}{x} \leq D^2$.
    \item[(A2)]\label{app:asp:a2} $F$ is $L$-Lipschitz smooth i.e.\ $\|F(x)-F(y)\|_* \leq L \|x-y\|$.  
    \item[(A3)]\label{app:asp:a3} For all $t$, $\|m_t\|_{*}  \leq G'$ and $\|g_t\|_{*}  \leq G'$.
    \item[(A4)]\label{app:asp:a4} For all $x \in \sx$, $\|F(x)\|_{*}  \leq G'$.
    \item[(B1)]\label{app:asp:b1} $F$ is $L_{\beta}$-\emph{Bregman smooth} i.e.\ $\|F(x)-F(y)\|_{x,*} \leq L_{\beta} \sqrt{2\Breg{x}{y}}$
    \item[(C1)]\label{app:asp:c1} $F$ is $M$ \emph{Bregman bounded} i.e.\ $\|F(x)\|_* \leq \frac{M\sqrt{\Breg{y}{x}}}{\|x-y\|}$ 
\end{description}

\section{Geometry-Aware Universal MP for Lipschitz \emph{Smooth/Bounded} Operator}
\label{app:GAUMP}
\begin{lem}
\label{app:lem:gap_update_1}
Consider the mirror-prox update. Then for any $z \in \sx$, we have 
\begin{equation}
    \label{app:eq:gap_update_1}
    \dpr{g_t}{\xt - z} \leq \inv{\lr_t}\left(\Breg{z}{\yn} - \Breg{z}{\yt} - \Breg{\xt}{\yn} - \Breg{\yt}{\xt} \right) + \dpr{\xt-\yt}{g_t - m_t} 
\end{equation}
\end{lem}
\begin{proof}
\begin{align}
    \label{app:eq:gap_eq}
     & \dpr{g_t}{\xt-\x} = \dpr{\xt-\yt}{g_t-m_t} + \underbrace{\dpr{\xt-\yt}{m_t}}_{A} + \underbrace{\dpr{\yt - z}{g_t}}_{B} 
\end{align}
To bound $A$ and $B$ we use Lem.~\ref{app:lem:three_point_ineq}. 
\begin{align}
        &  A \leq 1/\eta_t \left(\Breg{\yt}{\yn} -\Breg{\yt}{\xt} - \Breg{\xt}{\yn} \right)\\
   & B \leq 1/\eta_t \left(\Breg{z}{\yn} -\Breg{z}{\yt} - \Breg{\yt}{\yn}\right)
\end{align}
Placing the bound for $A$ and $B$ in~\ref{app:eq:gap_eq} gives us the  required result.   
\end{proof}
\subsection{Lipschitz Smooth Operator}
\label{app:Lsmooth_GAUMP}
In this subsection we assume that $F$ is $L$-Lipschitz smooth. 
\begin{lem}
\label{app:lem:bnd_z}
Let $\lr_t = \frac{D}{\sqrt{G_0^2 + \sum_{i=1}^{t-1} Z_i^2}}$ where $Z_t^2 = \frac{\Breg{\xt}{\yn}}{c^2\lr_t^2}$ and $c$ is a constant. Then there exist a constant $G=\frac{G'}{c}$ such that $Z_t \in [0,G]$. 
\end{lem}
\begin{proof}
Since $\xt$ is the optimum in the update~\ref{app:alg:mp_x_update}, we have 
\begin{align}
     \lr_t \dpr{m_t}{\xt} + \Breg{\xt}{\yn} \leq \lr_t \dpr{m_t}{\yn}
\end{align}
By rearranging the terms and dividing by $\lr_t$ we get 
\begin{align}
    \label{app:eq:upbd_on_z}
   c^2 Z_t^2 = \frac{\Breg{\xt}{\yn}}{\lr_t^2} \leq \frac{\dpr{m_t}{\yn-\xt}}{\lr_t} \leq\|m_t\|_{*} \frac{\|\xt - \yn\|}{\lr_t} \leq  G'\underbrace{\frac{\|\xt - \yn\|}{\lr_t}}_{A}
\end{align}
where $c'$ and $G'$ are constant. Now we need to bound $A$. Due to optimality condition we have
\begin{align}
    \dpr{\lr_t m_t+ \nabla R(\xt) - \nabla R(\yn) }{\yn-\xt} \geq 0. 
\end{align}
Therefore by rearranging and using the strong convexity of $R$ with $\mu=1$ , we get 
\begin{align}
     \dpr{\lr_t m_t}{\yn-\xt} \geq  \dpr{\nabla R(\xt) - \nabla R(\yn) }{\xt-\yn} \geq \|\xt-\yn\|^2
\end{align}
Finally by using the bound on $\|m_t\|_*$ we have 
\begin{align}
    \|\xt-\yn\|^2 \leq \dpr{\lr_t m_t}{\yn-\xt} \leq \|m_t\|_* \|\xt-\yn\| \lr_t \leq G' \|\xt-\yn\| \lr_t.
\end{align}
By dividing both side of above by $\|\xt-\yn\| \lr_t$ we get that $A \leq G'$. Now we use this upper bound in~(\ref{app:eq:upbd_on_z}) $Z_t^2 \leq c^{-2}G'^2 = G^2 $ which gives us the required result. 
\end{proof}

\begin{thm}
\label{app:thm:lsmooth}
\textbf{[Thm.~\ref{thm:lsmooth} in the main text.]} Assume assumptions \textbf{(A1-3)} holds. Then if we set $ \lr_t = \frac{D}{\sqrt{G_0^2 + \sum_{i=1}^{t-1} Z_i^2}}$ where $Z_t^2 = \frac{\Breg{\xt}{\yn}}{c^2\lr_t^2}$, we have 
\begin{equation}
    \label{app:eq:bnd_regret_lsmh}
     \text{Regret} \leq C.
\end{equation}
where $C$ is a constant dependent on $L,G',G_0$ and $D$ and $c^2=2$. 
\end{thm}
\begin{proof}
We first bound the with bounding $\dpr{\xt-\yt}{g_t - m_t}$ using Young inequality 
\begin{align}
    \dpr{\xt-\yt}{g_t-m_t} &\leq \frac {\lr_t^{-1}}{2} \|\xt - \yt\|^2 + \frac {\lr_t }{2} \|m_t - g_t\|_*^2 \\
   & \leq \inv{\lr_t} \Breg{\yt}{\xt} + \frac {\lr_t }{2} \|m_t - g_t\|_*^2\\
   &\leq \inv{\lr_t} \Breg{\yt}{\xt} + L^2\lr_t \|\xt - \yn\|^2 \\
   &\leq \inv{\lr_t} \Breg{\yt}{\xt} + \lr_t L^2 \Breg{\xt}{\yn}
\end{align}
where in the second inequality we use the property of Bregman divergence, in the third inequality we use the Lipschitz smoothness of $F$ and in the last inequality is due to the property of Bregman divergence. Then using Lem.~\ref{app:lem:gap_update_1} with $z=\x$ and above bound we get 
\begin{align}
\label{app:eq:gap_up_smooth}
    \dpr{g_t}{\xt-\x} \leq \inv{\lr_t} \left( \Breg{\x}{\yn} -\Breg{\x}{\yt}  \right)+ (L^2 \lr_t - \inv{\lr_t}) \Breg{\xt}{\yn}. 
\end{align}
If we sum up the above for $T$ iteration we get 
\begin{align}
    \text{Regret} &\leq \sum_{t=1}^T \dpr{g_t}{\xt-\x}\\
    &\leq  \frac{\Breg{\x}{y_1}}{\lr_1} - \frac{\Breg{\x}{y_T}}{\lr_T} \\
    &+ \sum_{t=2}^{T-1} \Breg{\x}{\yn}  \left( \inv{\lr_{t+1}} - \inv{\lr_{t}} \right) + \sum_{t=1}^T(L^2 \lr_t - \inv{\lr_t}) \Breg{\xt}{\yn}\\
    &\leq \frac{D^2}{\lr_1} + D^2 \sum_{t=2}^{T-1} \left( \inv{\lr_{t+1}} - \inv{\lr_{t}} \right)+ \sum_{t=1}^T(L^2 \lr_t - \inv{\lr_t}) \Breg{\xt}{\yn}\\
    &\leq \frac{D^2}{\lr_1} + \frac{D^2}{\lr_T} -\frac{D^2}{\lr_2} + \sum_{t=1}^T(L^2 \lr_t - \inv{\lr_t}) \Breg{\xt}{\yn}\\
     &\leq \frac{D^2}{\lr_1} + \frac{D^2}{\lr_T} + \sum_{t=1}^T(L^2 \lr_t - \inv{\lr_t}) \Breg{\xt}{\yn}\\
     &= \frac{D^2}{\lr_1} + \frac{D^2}{\lr_T} +c^2\sum_{t=1}^T(L^2 \lr_t - \inv{\lr_t}) \lr_t^2 Z_t^2.
\end{align}
The first inequality is based on the definition, the second one is just rearranging the terms, the third one is due to the positivity of Bregman divergence and the upperbound on $\Breg{\x}{x}$. The last equality is due to the definition of $Z_t$. To simplify the last term let define $\tau^*$ as follows 
\begin{align}
    \tau^* = \max \{t : \inv{\lr_t} \leq \sqrt{2}L\}.   
\end{align}
Therefore we can rewrite the last term as follows 
\begin{align}
    \sum_{t=1}^T(L^2 \lr_t - \inv{\lr_t}) \lr_t^2 Z_t^2 &=\sum_{t=1}^{\tau^*-1} (L^2 - \lr_t^{-2}) \lr_t^3 \zt^2+ \sum_{t=\tau^*}^T (L^2 - \lr_t^{-2})\lr_t^3 \zt^2\\
    &\leq \sum_{t=1}^{\tau^*-1} L^2 \lr_t^3 \zt^2 + \sum_{t=\tau^*}^T (L^2 - \lr_t^{-2})\lr_t^3 \zt^2\\
    & \leq \sum_{t=1}^{\tau^*-1} L^2 \lr_t^3 \zt^2 -1/2 \sum_{t=\tau^*}^T  ( \lr_t^{-2})\lr_t^3 \zt^2\\
    & \leq \sum_{t=1}^{\tau^*-1} L^2 \lr_t^2 \zt^2 -1/2 \sum_{t=\tau^*}^T  ( \lr_t^{-2})\lr_t^3 \zt^2. 
\end{align}
The first and second inequalities are due to the definition of $\tau^*$ and that $\lr_t > 0$. For the last inequality since $\lr_t \to 0$, we can assume $\lr_t$ is bounded and w.l.o.g.\ we assumed $\lr_t \leq 1$. If we put everything together we have 
\begin{align}
 \text{Regret} \leq \frac{D^2}{\lr_1} + \underbrace{\frac{D^2}{\lr_T}-c^2/2 \sum_{t=\tau^*}^T  \lr_t \zt^2}_{A} + \underbrace{c^2\sum_{t=1}^{\tau^*-1} L^2 \lr_t^2 \zt^2}_{B}. 
\end{align}
We first bound $A$. 
\begin{align}
    A &:= D\sqrt{G_0^2 + \sum_{t=1}^T \zt^2} - \frac{Dc^2}{2} \sum_{t=\tau^*}^{T-1} \frac{\zt^2}{\sqrt{G_0^2 + \sum_{i=1}^{t-1} Z_i^2}} \\
    &\leq DG_0 + D \sum_{t=1}^{T} \frac{\zt^2}{\sqrt{G_0^2 + \sum_{i=1}^{t-1} Z_i^2}} - \frac{Dc^2}{2} \sum_{t=\tau^*}^{T} \frac{\zt^2}{\sqrt{G_0^2 + \sum_{i=1}^{t-1} Z_i^2}}\\
    & \leq DG_0+ D\sum_{t=1}^{\tau^*} \frac{\zt^2}{\sqrt{G_0^2 + \sum_{i=1}^{t-1} Z_i^2}}\\
    & \leq DG_0+ \frac{2DG^2}{G_0}+ 3DG + 3D \sqrt{G_0^2 + \sum_{t=1}^{\tau^*-1} Z_t^2}\\
    &= DG_0+ \frac{2DG^2}{G_0}+ 3DG + 3D^2/\lr_{\tau^*}\\
    & \leq DG_0+ \frac{2DG^2}{G_0}+ 3DG + 3\sqrt{2}LD^2 := \rho_1(D,G_0,G,L)\\
\end{align}
where the first inequality comes from the LHS of  Lem.~\ref{app:lem:inv_sqrt_pos_sum} with $a_0 = G_0^2$ and $a_i = Z_i^2$. For the second inequality we set $c^2=2$. For the 3rd one we used the RHS of  Lem.~\ref{app:lem:inv_sqrt_pos_sum} and the boundedness of $Z_t$ due to  Lem.~\ref{app:lem:bnd_z}. Finally the last inequality is due to the definition of $\tau^*$ i.e.\ $\inv{\lr_{\tau^*}} \leq \sqrt{2} L $.
To bound B we use Lem.~\ref{app:lem:inv_pos_sum} with $a_0 = G_0^2$ and $a_i=Z_i^2$ besides the boundedness of $Z_t$. 
\begin{align}
    B&:= 2\sum_{t=1}^{\tau^*-1} L^2 \lr_t^2 \zt^2 = 2D^2L^2 \sum_{t=1}^{\tau^*-1} \frac{\zt^2}{G_0^2 + \sum_{i=1}^{t-1}Z_i^2}\\
    &\leq 2D^2L^2 \left( 2 + 4G^2/G_0^2 + 2 \log(1+\sum_{t=1}^{\tau^*-1}\frac{\zt^2}{G_0^2}) \right)\\
    &\leq 2D^2L^2 \left( 2 + 4G^2/G_0^2 + 2 \log(\sum_{t=1}^{\tau^*-1}\frac{G_0^2+ \zt^2}{G_0^2} )\right)\\
    &\leq 2D^2L^2 \left( 2 + 4G^2/G_0^2 +\log(\frac{D^2}{G_0^2 \lr_{\tau^*}^2})\right)\\
    &\leq 2D^2L^2 \left( 2 + 4G^2/G_0^2 +\log(\frac{2D^2L^2}{G_0^2})\right):=\rho_2(D,L,G_0,G).
\end{align}
So finally we have the following bound for the $\text{Regret}$ where $\rho_1$ and $\rho_2$ are polynomial functions w.r.t.\ their inputs. 
\begin{equation}
    \text{Regret} \leq \frac{D}{\eta_1} + \rho_1(D,G_0,G,L)+ \rho_2(D,L,G_0,G). 
\end{equation}
\end{proof}

\subsection{Lipschitz Bounded operator}
\label{app:Lbounded_GAUMP}
In this subsection we assume that $F$ is not Lipschitz smooth but bounded i.e.\ for all $x \in \sx $ we have $\|F(x)\|_*\leq G'$. 

\begin{lem}
\label{app:lem:bnd_z_lbndop}
Let $\lr_t = \frac{D}{\sqrt{G_0^2 + \sum_{i=1}^{t-1} Z_i^2}}$ where $Z_t^2 = \frac{\Breg{\xt}{\yn}+\Breg{\yt}{\xt}}{c^2\lr_t^2}$ and $c$ is a constant. Then there exist a constant $G=\frac{\sqrt{3}G'}{c}$ such that $Z_t \in [0,G]$.
\end{lem}
\begin{proof}
Based on~\ref{app:lem:bnd_z}, we know that  $\frac{\Breg{\xt}{\yn}}{\lr_t^2} \leq G'^2$. So we just need to bound the term $\frac{\Breg{\yt}{\xt}}{\lr_t^2}$. We start by bounding  $\frac{\Breg{\yt}{\yn}}{\lr_t^2}$. Based on~\eqref{app:alg:mp_y_update}, we get 
\begin{align}
    \lr_t \dpr{\yt}{g_t} + \Breg{\yt}{\yn}\leq \lr_t \dpr{\yn}{g_t} 
\end{align}
By rearranging and dividing by $\lr_t^2$ we have 
\begin{align}
    \frac{\Breg{\yt}{\yn}}{\lr_t^2} \leq \frac{\dpr{\yn-\yt}{g_t} }{\lr_t} \leq G' \frac{\|\yn-\yt\|}{\lr_t}.
\end{align}
Similar to the proof of Lem.~\ref{app:lem:bnd_z} for bounding $\|\xt - \yn\|/\lr_t$, we can show that $\frac{\|\yn-\yt\|}{\lr_t} \leq G'$ for some constant $G'$. Now we use Lem.~\ref{app:lem:three_point_ineq} with $x^+= \xt, p=\yt, x=\yn$ and $d=\lr_t m_t$ we get 
\begin{align}
    \lr_t \dpr{\xt-\yt}{m_t} &\leq \Breg{\yt}{\yn} - \Breg{\yt}{\xt} - \Breg{\xt}{\yn}\\
    &\leq \Breg{\yt}{\yn} - \Breg{\yt}{\xt}.
\end{align}
By rearranging and dividing by $\lr_t^2$ we get 
\begin{align}
    \frac{\Breg{\yt}{\xt}}{\lr_t^2} &\leq \frac{\Breg{\yt}{\yn}}{\lr_t^2} + \frac{\dpr{\yt-\xt}{m_t}}{\lr_t}\\
    &\leq G'^2 + \|m_t\|_* \frac{\|\yt-\xt\|}{\lr_t} \\
    &\leq G'^2 + 2G'^2, 
\end{align}
where the first inequality is due to bound of $\frac{\Breg{\yt}{\yn}}{\lr_t^2}$. The second inequaolity is due to triangular inequality i.e.\ $\|\yt-\xt\| \leq \|\yt -\yn \|  + \| \xt-\yn \|$. So we $\frac{\|\yt-\xt\|}{\lr_t} \leq 2G'$. 

\end{proof}

\begin{thm}
\label{app:thm:lpz_bndd}
\textbf{[Thm.~\ref{thm:lbounded} in the main text.]} Assume \textbf{(A1),(A4)} holds. If we set $\lr_t= \frac{D}{\sqrt{G_0^2 + \sum_{i=1}^{t-1}Z_i^2}}$
where $Z_t^2 = \frac{\Breg{\xt}{\yn}+\Breg{\yt}{\xt}}{c^2\lr_t^2}$ we can bound the regret as follows
\begin{equation}
    \text{Regret} \leq O(\sqrt{T \log(T)} )
\end{equation}
where $c=1$.  
\end{thm}

\begin{proof}
We first bound the with bounding $\dpr{\xt-\yt}{g_t - m_t}$.
\begin{align}
    \dpr{\xt-\yt}{g_t - m_t} &\leq \|g_t - m_t\|_*\|\xt-\yt\| \leq 2G'\|\xt-\yt\|\\
    &\leq 4G'\sqrt{\Breg{\yt}{\xt}} \leq 4G'\sqrt{\Breg{\yt}{\xt}+ \Breg{\xt}{\yn}}\\
    &\leq 4cG'\lr_t\zt
\end{align}
where we use the bound on norm of $F$ and the definition of $\zt$ and $\lr_t$ and the property of Bregman divergence. Now using the above bound in~\ref{app:eq:gap_update_1} with $z=\x$ we get
\begin{align}
 \dpr{g_t}{\xt - z} &\leq \inv{\lr_t}\left(\Breg{\x}{\yn} - \Breg{\x}{\yt} \right)\\
  & - \inv{\lr_t}\left(\Breg{\xt}{\yn} + \Breg{\yt}{\xt} \right) + 4cG'\lr_t\zt\\
  &= \inv{\lr_t}\left(\Breg{\x}{\yn} - \Breg{\x}{\yt} \right) - \inv{\lr_t} (c^2\lr_t^2\zt^2) + 4cG'\lr_t\zt
\end{align}
If we sum up the above for $T$ iteration and follwoing the steps similar to the proof of previous theorem we have 
\begin{align}
    \text{Regret} \leq \underbrace{\frac{D^2}{\lr_1} + \frac{D^2}{\lr_T}}_{A} - c^2\underbrace{\sum_{t=1}^T \lr_t \zt^2 }_{B}+ 4cG'\underbrace{\sum_{t=1}^T\lr_t\zt}_{C} 
\end{align}
To bound $A$ we have 
\begin{align}
    A:= DG_0 + D \sqrt{G_0^2 + \sum_{t=1}^{T-1} Z_t^2}. 
\end{align}
Then we find lower bound for $B$
\begin{align}
    B &:= \sum_{t=1}^T \lr_t \zt^2 = D\sum_{t=1}^T \frac{\zt^2}{\sqrt{G_0^2 + \sum_{i=1}^{t-1} Z_i^2}}\\
    &\geq D\sqrt{G_0^2 + \sum_{t=1}^{T-1} Z_t^2} - D G_0 
\end{align}
where we use the L.H.S. of Lem.~\ref{app:lem:inv_sqrt_pos_sum} with $a_0=G_0$ and $a_i = Z_i^2$. To bound $C$ we use Lem.~\ref{app:lem:inv_pos_sum}. 
\begin{align}
    C&:= \sum_{t=1}^T\lr_t\zt \leq \sqrt{T}\sqrt{\sum_{t=1}^T \lr_t^2 \zt^2 } \label{app:eq:lr_zt} \\
    &\leq D \sqrt{T}\sqrt{\sum_{t=1}^T \frac{\zt^2}{G_0^2 + \sum_{i=1}^{t-1} Z_i^2} } \leq  D \sqrt{T}\sqrt{(2+ \frac{4G}{G_0^2}+2\log(1+\sum_{t=1}^{T-1} \frac{\zt^2}{G_0^2}))}\\
    & \leq D \sqrt{T}\sqrt{(2+ \frac{4G^2}{G_0^2}+2\log(1+ \frac{T G^2}{G_0^2}))}.
\end{align}
Note that for the last inequality we used the fact that $\zt \leq G$. This can be shown similar to what we show in Lem.~\ref{app:lem:bnd_z}. Finally setting $c=1$ and replacing $A$, $B$, and $C$ with their corresponding bound we have 
\begin{align}
    \text{Regret} \leq 2 DG_0 + 4G'D \sqrt{T}\sqrt{(2+ \frac{4G^2}{G_0^2}+2\log(1+ \frac{T G^2}{G_0^2}))} = O(\sqrt{T \log(T)} ) 
\end{align}
\end{proof}

\section{Convergence for Bregman Smooth Operator}
\label{app:Bsmooth_GAUMP}
In this section we assume that $F$ is not \emph{Lipschitz smooth} but \emph{Bregman smooth} with parameter $L_{\beta}$. 
\begin{thm}
\textbf{[Thm.~\ref{thm:Bsmooth} in the main text.]} Assume assumptions \textbf{(A1),(A3)} and \textbf{(B1)} hold. Then if we set $ \lr_t = \frac{D}{\sqrt{G_0^2 + \sum_{i=1}^{t-1} Z_i^2}}$ where $Z_t^2 = \frac{\Breg{\xt}{\yn}}{c^2\lr_t^2}$, we have 
\begin{equation}
    \label{app:eq:bnd_regret_bsmh}
     \text{Regret} \leq C.
\end{equation}
where $C$ is a constant dependent on $L_{\beta},G',G_0$ and $D$ and $c^2=2$
\end{thm}
\begin{proof}
The proof is very similar to the proof of Thm.~\ref{app:thm:lsmooth}. The main difference is the change in the upper bound of $\dpr{\xt-\yt}{g_t - m_t}$. 
\begin{align}
   \dpr{\xt-\yt}{g_t-m_t} &\leq \frac {\lr_t^{-1}}{2} \|\xt - \yt\|_{\xt}^2 + \frac {\lr_t }{2} \|m_t - g_t\|_{x_t,*}^2 \\
   & \leq \inv{\lr_t} \Breg{\yt}{\xt} + \frac {\lr_t }{2} \|m_t - g_t\|_{x_t,*}^2\\
   &\leq \inv{\lr_t} \Breg{\yt}{\xt} + \lr_t L_{\beta}^2 \Breg{\xt}{\yn}
\end{align}
where the first inequality is due to the Young inequality with function $1/2\|.\|_{x_t}$, the second one is based on the definition of Bregman divergence and the last one is due to the definition of \emph{\emph{Bregman smooth}ness}. Finally similar to~\eqref{app:eq:gap_up_smooth} of the proof of theorem~\ref{app:thm:lsmooth} we have 
\begin{align}
     \dpr{g_t}{\xt-\x} \leq \inv{\lr_t} \left( \Breg{\x}{\yn} -\Breg{\x}{\yt}  \right)+ (L_{\beta}^2 \lr_t - \inv{\lr_t}) \Breg{\xt}{\yn}. 
\end{align}
The rest of the proof is the same as the proof of Thm.~\ref{app:thm:lsmooth} except we replace $L$ with $L_{\beta}$.  
\end{proof}

\section{Convergence for Bregman Bounded Operator}
\label{app:Bbounded_GAUMP}
In this section we assume that $F$ is relative bounded i.e.\ assumption \textbf{(C1)} holds. First we prove the modified version of Lem.~\ref{app:lem:bnd_z}. 
\begin{lem}
\label{app:lem:bnd_z_rel_bnd}
 Let $\lr_t = \frac{D}{\sqrt{G_0^2 + \sum_{i=1}^{t-1} Z_i^2}}$ where $Z_t^2 = \frac{\Breg{\xt}{\yn}+\Breg{\yt}{\xt}}{c^2\lr_t^2}$ and $c$ is a constant. Then there exist a constant $G=\frac{M+\sqrt{6M+3M\sqrt{M}+4M^2}}{c}$ such that $Z_t \in [0,G]$. 
\end{lem}
\begin{proof}
  We first show that $\frac{\sqrt{\Breg{\xt}{\yn}}}{\lr_t}$ is bounded. Since $\xt$ is the optimum in the update~\ref{app:alg:mp_x_update}, we have 
\begin{align}
     \lr_t \dpr{m_t}{\xt} + \Breg{\xt}{\yn} \leq \lr_t \dpr{m_t}{\yn}. 
\end{align}
  By rearranging we get 
\begin{align}
    \frac{\Breg{\xt}{\yn}}{\lr_t} \leq  \dpr{m_t}{\yn-\xt}\leq \|m_t\|_* \|\xt - \yn\| \leq M \sqrt{\Breg{\xt}{\yn}}. 
\end{align}
So we have 
\begin{align}
    \frac{\sqrt{\Breg{\xt}{\yn}}}{c\lr_t} \leq \frac{M}{c} = G_1.               
\end{align}

Similarly we can get an upper-bound for $\frac{\sqrt{\Breg{\yt}{\xt}}}{c\lr_t} \leq G_2$. So we have 
\begin{align}
    Z_t = \frac{\sqrt{ \Breg{\xt}{\yn} +\Breg{\yt}{\xt} }}{c\lr_t} &\leq \frac{\sqrt{ \Breg{\xt}{\yn}} +\sqrt{\Breg{\yt}{\xt} }}{c\lr_t} \\
    & \leq G_1 + G_2 = G.  
\end{align}
To finish the proof, we need to show that $\frac{\sqrt{\Breg{\yt}{\xt}}}{c\lr_t}$ is bounded. To simplify the notation, let $\alpha_t = \frac{\Breg{\yt}{\xt}}{\lr_t^2},\beta_t = \frac{\Breg{\yt}{\yn}}{\lr_t^2}$ and $\gamma_t = \frac{\Breg{\xt}{\yn}}{\lr_t^2}$. We start by using Lem.~\ref{app:lem:three_point_ineq} with $x^+=\xt$, $p=\yt$, $x=\yn$ and $d=\lr_t m_t$ and then dividing both sides by $\lr_t^2$ and rearranging the terms 
\begin{align}
    \alpha_t + \gamma_t &\leq \beta_t + \frac{\dpr{m_t}{\yt-\xt}}{\lr_t}= \beta_t + \frac{\dpr{m_t}{\yt-\yn}}{\lr_t} + \frac{\dpr{m_t}{\yn-\xt}}{\lr_t}\\
    & \leq \beta_t + \frac{\|m_t\|_*\|\yt-\yn\|}{\lr_t} + \frac{\|m_t\|_*\|\yn-\xt\|}{\lr_t}\\
    & \leq \beta_t + \frac{1}{\lr_t}\left\{ M \sqrt{\Breg{\yt}{\yn}} + M \sqrt{\Breg{\xt}{\yn}}   \right\}\\
    & \leq \beta_t + \frac{1}{\lr_t}\left\{ \frac{a \lr_tM^2}{2}+  \frac{\Breg{\yt}{\yn}}{2a\lr_t} + \frac{b \lr_tM^2}{2} + \frac{\Breg{\xt}{\yn}}{2b\lr_t}   \right\}\\
    & = \beta_t + Z + \frac{\beta_t}{2a} + \frac{c_t}{2b} = (1+1/(2a))\beta_t + Z + \frac{\gamma_t}{2b},
\end{align}
where last equality is due to Young inequality and $a$ and $b$ are some positive scalars and $Z =\frac{a M^2}{2} + \frac{b M^2}{2}$ is a positive constant. Now we need to bound $\beta_t$. To do so, using Lem.~\ref{app:lem:three_point_ineq} with $x^+=\yt$, $p=\yn$, $x=\yn$ and $d=\lr_t g_t$, and dividing both side by $\lr_t^2$ and rearranging we have 
\begin{align}
    \beta_t &\leq \frac{\dpr{\yn-\yt}{g_t}}{\lr_t} = \frac{\dpr{\yn-\xt}{g_t}}{\lr_t} + \frac{\dpr{\xt-\yt}{g_t}}{\lr_t}\\
    & \leq \frac{1}{\lr_t}\left\{ \|\yn-\xt\| \|g_t\|_* + \|\yt-\xt\| \|g_t\|_*  \right\}\\
    & \leq \frac{1}{\lr_t}\left\{ M \sqrt{\Breg{\yn}{\xt}} + M \sqrt{\Breg{\yt}{\xt}} \right\}\\
    & \leq \frac{1}{\lr_t}\left\{ \frac{\lr_t d M}{2}+ \frac{\Breg{\yn}{\xt}}{2d\lr_t} + \frac{\lr_t e M}{2}+  \frac{\Breg{\yt}{\xt}}{2e\lr_t} \right\}\\
    & = Z' + \frac{\Breg{\yn}{\xt}}{2d\lr_t^2} + \frac{\alpha_t}{2e}, 
\end{align}
where $Z'= \frac{(d+e)M}{2}$, $d$ and $e$ are a positive constants. To bound $\frac{\Breg{\yn}{\xt}}{\lr_t^2}$, we use Lem.~\ref{app:lem:three_point_ineq} with $x^+=\xt$, $p=\yn$, $x=\yn$ and $d=\lr_t m_t$. Like above by dividing to $\lr_t^2$ and rearranging we get 
\begin{align}
    \frac{\Breg{\yn}{\xt}}{\lr_t^2} \leq \frac{\dpr{\yn-\xt}{m_t}}{\lr_t}\leq \frac{\|\yn-\xt\|\|m_t\|_*}{\lr_t} \leq M\sqrt{\gamma_t} = N 
\end{align} 
which gives us the required result. Putting everything together we have 
\begin{align}
    \alpha_t + \gamma_t &\leq (1+1/(2a))\left( Z' + \frac{N}{2d} + \frac{\alpha_t}{2e} \right) + Z + \frac{\gamma_t}{2b}\\
    &\leq H + \frac{(1+1/(2a))}{2e} \alpha_t + \frac{\gamma_t}{2b}.
\end{align}
Setting $a=e=b=d=1$ we have 
\begin{align}
    \frac{\alpha_t}{4} \leq \frac{\alpha_t}{4} + \frac{\gamma_t}{2} \leq H,   
\end{align}
which shows $\alpha_t = \frac{\Breg{\yt}{\xt}}{\lr_t^2}$ is bounded. 
\end{proof}

\begin{thm}
\label{app:thm:rel_bnd}
\textbf{[Thm.~\ref{thm:Bbounded} in the main text.]} Assume \textbf{(A1),(C1)} holds. If we set $\lr_t= \frac{D}{\sqrt{G_0^2 + \sum_{i=1}^{t-1}Z_i^2}}$
where $Z_t^2 = \frac{\Breg{\xt}{\yn}+\Breg{\yt}{\xt})}{c^2\lr_t^2}$ we can bound the regret as follows
\begin{equation}
    \text{Regret} \leq \mathcal{O}(\sqrt{T\log(T)})
\end{equation}
where $c=1$.  
\end{thm}

\begin{proof}
We begin by bounding $\dpr{g_t}{\xt -\x}$. 
\begin{align}
    \label{app:eq:gap_bnd_2}
    \dpr{g_t}{\xt -\x} = \underbrace{\dpr{g_t}{\yt-\x}}_{A} + \underbrace{\dpr{g_t}{\xt-\yt}}_{B}  
\end{align}
Using Lem.~\ref{app:lem:three_point_ineq} we can bound $A$. 
\begin{align}
    A:= \dpr{g_t}{\yt-\x} &\leq 1/\eta_t \left(\Breg{\x}{\yn} -\Breg{\x}{\yt} - \Breg{\yt}{\yn}\right)\\
    & \leq 1/\eta_t \left(\Breg{\x}{\yn} -\Breg{\x}{\yt} \right).
\end{align}
Using the relative boundedness we bound $B$
\begin{align}
    B:= \dpr{g_t}{\xt-\yt} &\leq \|g_t\|_* \|\xt-\yt\| \leq M'\sqrt{\Breg{\yt}{\xt}}\\
    &= Mc\lr_t\zt 
\end{align}
Using the above upper bounds for $A$ and $B$ and summing up~\eqref{app:eq:gap_bnd_2} for $T$ iterations we get 
\begin{align}
    \text{Regret} &\leq \sum_{i=1}^T \dpr{g_t}{\xt -\x}\\
    &\leq \underbrace{\frac{D^2}{\lr_1} + \frac{D^2}{\lr_T}}_{U} + M c \underbrace{\sum_{i=1}^T \lr_t\zt}_{V}    
\end{align}
First we bound $A$ by using the definition of $\lr_T$ and the fact that $Z_t \leq G$ and setting $c=1$. 
\begin{align}
    U:= \frac{D^2}{\lr_1} + \frac{D^2}{\lr_T} &= DG_0 + D\sqrt{G_0^2 + \sum_{t=1}^{T-1} Z_t^2}\\ 
    &\leq DG_0 + D\sqrt{G_0^2 + \sum_{t=1}^{T-1} Z_t^2} \leq   DG_0 + D\sqrt{G_0^2 + T G^2} = \mathcal{O}(\sqrt{T})
\end{align}
To bound $B$ we leverage the result of Lem.~\ref{app:lem:inv_pos_sum}. 
\begin{align}
    V:= \sum_{t=1}^T \lr_t\zt &=\sum_{t=1}^T \sqrt{\lr_t^2\zt^2}\\
    &\leq \sqrt{T} \sqrt{\sum_{t=1}^T \lr_t^2\zt^2}\\
    &= \sqrt{TD} \sqrt{\sum_{t=1}^T \frac{\zt^2}{G_0^2 +\sum_{i=1}^{t-1} Z_i^2}}\\
    & \leq \sqrt{TD}\sqrt{2+ \frac{4G^2}{G_0^2}+ 2\log(1+\sum_{t=1}^{T-1}\frac{\zt^2}{G_0^2})}\\
    & \leq \sqrt{TD}\sqrt{2+ \frac{4G^2}{G_0^2}+ 2\log(1+\sum_{t=1}^{T-1}\frac{G^2}{G_0^2})}\\
    & \leq \mathcal{O}(\sqrt{T\log(T)})
\end{align}
Finally by adding the upper bound for $A$ and $B$ we get 
\begin{align}
    \text{Regret} \leq \mathcal{O}(\sqrt{T}) + \mathcal{O}(\sqrt{T\log(T)}) = \mathcal{O}(\sqrt{T\log(T)}). 
\end{align}
\end{proof}


\section{Convergence for Stochastic Setting}
\label{app:Stoch_GAUMP}
In this section we consider the stochastic version of the MP for different settings. We assume that $\tilde{g_t}$ and $\tilde{m_t}$ to be the noisy version of $g_t$ and $m_t$ and also $\mathbb E [\tilde{g}_t] =g_t$ and $\mathbb E [\tilde{m}_t ]=m_t$. Besides we make the following assumption
\begin{description}
\item[D1] $\mathbb E \left[ \|\tilde{g}_t - g_t\|^2 \right] \leq \sigma^2$ and $\mathbb E \left[ \|\tilde{m}_t - m_t\|^2 \right] \leq \sigma^2$
\item[D2] $\|\tilde F\|_* \leq G'$
\item[D3] $ \|\tilde g(x)\|_*\leq M \frac{\sqrt{\Breg{y}{x}}}{\|x-y\|}$
\end{description}
The proof in this part is based on the technique developed in \citet{bach2019universal}. So we are not repeating their technique and just mention the high level steps. For the rest of this section, recall that $\Delta$ is the gap function.


\subsection{Lipschitz Smooth Setting}
\label{app:Stoch_Lsmooth_GAUMP}
\begin{thm}
\label{app:thm:stoc_lp_smt}
\textbf{[Thm.~\ref{thm:stoc_lip_smt} in the main text.]} Assume \textbf{(A1-3)} and \textbf{(D1)}. If we set $\lr_t= \frac{D}{\sqrt{G_0^2 + \sum_{i=1}^{t-1}Z_i^2}}$ where $Z_t^2 = \frac{\Breg{\xt}{\yn}+\Breg{\yt}{\xt})}{c^2\lr_t^2}$ we have 
\begin{align}
    \mathbb E \max_{x \in \sx } \Delta(\bar{x}_T, x) \leq \mathcal O (\sqrt{\log(T)}/\sqrt{T}).
\end{align}
where $c=5$. 
\end{thm}
\begin{proof}
By replacing noisy version of operator values in Lem.~\ref{app:lem:gap_update_1}
\begin{align*}
    \dpr{\tgt}{\xt - z} \leq &  \inv{\lr_t}\left(\Breg{z}{\yn} - \Breg{z}{\yt} - \Breg{\xt}{\yn} - \Breg{\yt}{\xt} \right) \\&
    + \dpr{\xt-\yt}{\tgt - \tmt}\\
    &= \inv{\lr_t}\left(\Breg{z}{\yn} - \Breg{z}{\yt} - \Breg{\xt}{\yn} - \Breg{\yt}{\xt} \right)\\
    & + \dpr{\xt-\yt}{g_t - m_t} + \dpr{\xt-\yt}{\gamma_t} 
\end{align*}
where $\gamma_t = (\tgt-g_t) +(\tmt - m_t) $ is a martingale difference sequence. 
\begin{align}
    \dpr{\xt-\yt}{g_t - m_t} \leq & \frac {L}{2} \|\xt - \yt\|^2 + \frac {1 }{2L} \|m_t - g_t\|_*^2 \\
   & \leq L \Breg{\yt}{\xt} + \frac {1 }{2L} \|m_t - g_t\|_*^2\\
   &\leq L\Breg{\yt}{\xt} + \frac {L}{2}\|\xt - \yn\|^2 \\
   &\leq L\Breg{\yt}{\xt} + L\Breg{\xt}{\yn}
\end{align}
With above inequality and taking the same steps as in Thm~\ref{app:thm:lsmooth} we have \begin{align}
    \sum_{t=1}^T \dpr{\tgt}{\xt - z} \leq & \underbrace{\frac{D^2}{\lr_1} + \frac{D^2}{\lr_T} +c\sum_{t=1}^T(L - \inv{\lr_t}) \lr_t^2 Z_t^2}_{A} + \sum_{t=1}^T  \dpr{\xt-\yt}{\gamma_t}.
\end{align}  
With $c=5$, similar to the proof in Thm~\ref{app:thm:lsmooth}, we can show that $A \leq \tilde{\rho}(D,G_0,G,L,\lr_1)$ where $\tilde{\rho}$ is a polynomial function of its inputs where $G$ is the upperbound for $Z_t$. We can show that $Z_t$ is bounded in a similar way to Lem.~\ref{app:lem:bnd_z_rel_bnd}.   
Let $\x =\displaystyle\argmax_{x\in \sx} \Delta(\bar{x}_T,x)$   
\begin{align*}
\mathbb E \left[\sum_{t=1}^T \dpr{g_t}{\xt - \x}\right] &= \mathbb E \left[\sum_{t=1}^T \dpr{\tgt}{\xt - \x}\right]+ \mathbb E \left[\sum_{t=1}^T \dpr{\xt-\yt}{\gamma_t}\right] - \mathbb E \left[\sum_{t=1}^T \dpr{\zeta_t}{\xt - \x}\right]\\
& \leq \tilde{\rho}(D,G_0,G,L,\lr_1) + \underbrace{\mathbb E \left[\sum_{t=1}^T \dpr{\zeta_t}{ \x}\right]}_{A}+\underbrace{\mathbb E \left[\sum_{t=1}^T \dpr{\xt-\yt}{\gamma_t}\right]}_{B} 
\end{align*}
We can bound \textbf{A} using Lem.~\ref{app:lem:mrtg_diff} by setting $X=\x*$ and $Z_t = \zeta_t$ and their upper bounds 
\begin{align}
    A \leq D/2 \sqrt{\sum_{t=1}^T \mathbb E \|\zeta_i\|_*^2} \leq \mathcal{O}(D\sigma\sqrt{T})
\end{align}
To bound \textbf{B} we start by using Cauchy-Schwartz and Jensen  inequalities
\begin{align}
    B \leq \mathbb E \sqrt{\sum_{t=1}^T \|\gamma_t\|_*^2} \sqrt{\sum_{t=1}^T \|\xt-\yt\|^2} \leq  \underbrace{\sqrt{\sum_{t=1}^T \mathbb E \|\gamma_t\|_*^2}}_{U} \underbrace{\sqrt{\mathbb E \sum_{t=1}^T\|\xt-\yt\|^2}}_{V} 
\end{align}
Due to assumption \textbf{(D1)} we get $U \leq \mathcal O(\sigma \sqrt{T})$.
\begin{align}
 V &\leq  \sqrt{\mathbb E \sum_{t=1}^T\|\xt-\yt\|^2 + \|\xt-\yn\|^2}\\
  &\leq \sqrt{2}\sqrt{\mathbb E \sum_{t=1}^T \Breg{\yt}{\xt} + \Breg{\xt}{\yn}}\\ 
  &\leq D\sqrt{10} \sqrt{\mathbb E \underbrace{\sum_{t=1}^T \lr_t^2 \zt^2}_{K}}
\end{align}
As we can see in~\eqref{app:eq:lr_zt}, we have $K \leq \mathcal O (\sqrt{\log(T)})$ and therefore $B \leq \mathcal O(\sqrt{T\log(T)})$. Putting everything together 
\begin{align}
    T.\mathbb E \Delta (\bar{x}_T, \x) \leq &\mathbb E \sum_{t=1}^T \Delta (\xt, x) \leq \mathbb E \sum_{t=1}^T \dpr{g_t}{\xt-x}\\
    & \leq \tilde{\rho}(D,G_0,G,L,\lr_1) + \mathcal{O}(D\sigma\sqrt{T}) + \mathcal O(\sqrt{T\log(T)})
\end{align}
which gives us the required result. 
\end{proof}

\subsection{Stochastic Bregman Smooth Setting}
\label{app:Stoch_Bsmooth_GAUMP}
\begin{cor}
\textbf{[Cor.~\ref{cor:stoc_Bsmooth} in the main text.]} Assume \textbf{(A1),(A3),(B1)} and \textbf{(D1)}. If we set $\lr_t= \frac{D}{\sqrt{G_0^2 + \sum_{i=1}^{t-1}Z_i^2}}$ where $Z_t^2 = \frac{\Breg{\xt}{\yn}+\Breg{\yt}{\xt})}{c^2\lr_t^2}$ we have 
\begin{align}
    \mathbb E \max_{x \in \sx }  \Delta(\bar{x}_T, x) \leq \mathcal O (\sqrt{\log(T)}/\sqrt{T}).
\end{align}
where $c=5$. 
\end{cor}
\begin{proof}
The proof is very similar to the proof of~\ref{app:thm:stoc_lp_smt}. The only part which changes is the following part \begin{align}
   \dpr{\xt-\yt}{g_t-m_t} &\leq \frac {L_{\beta}}{2} \|\xt - \yt\|_{\xt}^2 + \frac {1}{2L_{\beta}} \|m_t - g_t\|_{x_t,*}^2 \\
   & \leq L_{\beta} \Breg{\yt}{\xt} + \frac {1}{2L_{\beta}} \|m_t - g_t\|_{x_t,*}^2\\
   &\leq L_{\beta} \Breg{\yt}{\xt} + L_{\beta} \Breg{\xt}{\yn}
\end{align}
The rest of the proof is the same as~\ref{app:thm:stoc_lp_smt}.
\end{proof}


\subsection{Stochastic Lipschitz Bounded Setting}
\label{app:Stoch_Lbounded_GAUMP}
\begin{thm}
\label{app:thm:stoc_lip_bndd}
\textbf{[Thm.~\ref{thm:stoch_Lbounded} in the main text.]} Assume \textbf{(A1),(D2)} holds. With the same step-size as in Thm.~\ref{app:thm:lpz_bndd}, we get 
\begin{equation}
    \mathbb E \max_{x \in \sx }\Delta(\bar{x}_T, x) \leq \mathcal O (\sqrt{\log(T)}/\sqrt{T}).  
\end{equation}
\end{thm}
\begin{proof}
Replacing $g_t$ and $m_t$ with $\tilde{g}_t$ and $\tilde{m}_t$ and taking exactly the same step as in the proof of Thm.~\ref{app:thm:lpz_bndd} we get 
\begin{equation}
    \sum\dpr{\tilde{g}_t}{\xt-x} \leq \mathcal O(\sqrt{T\log(T)}). 
\end{equation}
Using above we have
\begin{align}
    T.\Delta (\tilde{x}_T, x) \leq &\sum_{t=1}^T \Delta (\xt, x) \leq \sum_{t=1}^T \dpr{g_t}{\xt-x}\\
    & = \sum_{t=1}^T \dpr{\tilde{g}_t}{\xt-x} - \sum_{t=1}^T \dpr{\zeta_t}{\xt-x}
\end{align}
where $\zeta_t = \tilde{g}_t - g_t$ is a martingale difference. Letting $\x = \displaystyle\argmax_{x\in \sx} \Delta(\bar{x}_T,x)$ we have 
\begin{align}
    T.\mathbb E \Delta (\bar{x}_T, \x) \leq & \mathcal{O} (\sqrt{T\log(T)}) - \mathbb E \left[ \sum_{t=1}^T \dpr{\zeta_t}{\xt-\x} \right]\\
    &  =\mathcal{O} (\sqrt{T\log(T)}) +\sum_{t=1}^T \mathbb E \left[  \dpr{\zeta_t}{\x} \right]\\
    & \leq \mathcal{O} (\sqrt{T\log(T)}) + \mathcal{O} (\sqrt{T})
\end{align}
where the last inequality comes from using Lem.~\ref{app:lem:mrtg_diff} with $Z_t =\zeta_t$ and $X=\x$ and also the fact that $\|\zeta_t\|_*\leq 2G'$.  
\end{proof}

\subsection{Stochastic Bregman Bounded Setting}
\label{app:Stoch_Bbounded_GAUMP}
\begin{thm}
\textbf{[Thm.~\ref{thm:stoch_Bbounded} in the main text.]} Assume \textbf{(A1),(D3)} hold. If we set $\lr_t= \frac{D}{\sqrt{G_0^2 + \sum_{i=1}^{t-1}Z_i^2}}$
where $Z_t^2 = \frac{\Breg{\xt}{\yn}+\Breg{\yt}{\xt})}{c^2\lr_t^2}$ we can bound the regret as follows
\begin{equation}
    \mathbb E \max_{x \in \sx }\Delta(\bar{x}_T, x) \leq \mathcal{O}(\frac{\sqrt{\log(T)}}{\sqrt{T}})
\end{equation}
where $G$ is the upper-bound for $Z_t$ and $c=1$.  
\end{thm}
\begin{proof}
Following the same steps as in Thm.~\ref{app:thm:rel_bnd} we get 
\begin{align}
    \sum\dpr{\tilde{g}_t}{\xt- x} \leq \mathcal O(\sqrt{T\log(T)} )
\end{align}
The rest of the proof is exactly the same as the Thm.~\ref{app:thm:stoc_lip_bndd}.
\end{proof}

\section{Auxiliary Lemmas}
\label{app:auxiliary}
In this section, we present lemmas which has been used in the proofs in the previous sections and also has been proved in the other papers. 

\begin{lem}
\label{app:lem:three_point_ineq}~\citep{antonakopoulos2019adaptive}
Assume that $x^+ =\displaystyle\argmin_{z \in \sx}\{ \dpr{d}{z} + \Breg{z}{x} \}$, where $R$ is a proper divergence inducing function and $x \in \sx$ and $d \in \sv$. Then for any $p \in \sx$ we have 
\begin{equation}
    \label{app:eq:three_point_ineq}
    \dpr{x^+-p}{d} \leq \Breg{p}{x} - \Breg{p}{x^+} -\Breg{x^+}{x} 
\end{equation}
\end{lem}

\begin{lem}
\label{app:lem:inv_sqrt_pos_sum}~\citep{bach2019universal}
Assume $a_1,a_2, ..., a_n \in [0,a]$ and $a_0 \geq 1$. Then 
\begin{equation}
    \label{app:eq:inv_sqrt_pos_sum}
    \sqrt{a_0 + \sum_{i=1}^{n-1} a_i } - \sqrt{a_0} \leq \sum_{i=1}^n \frac{a_i }{\sqrt{a_0 + \sum_{j=1}^{i-1} a_j }} \leq \frac{2a}{\sqrt{a_0}}+ 3\sqrt{a}+ 3\sqrt{a_0 + \sum_{i=1}^{n-1} a_i}  
\end{equation}
\end{lem}

\begin{lem}
\label{app:lem:inv_pos_sum}~\citep{bach2019universal}
Assume $a_1,a_2, ..., a_n \in [0,a]$ and $a_0 \geq 1$. Then 
\begin{equation}
    \label{app:eq:inv_pos_sum}
    \sum_{i=1}^n \frac{a_i }{a_0 + \sum_{j=1}^{i-1} a_j } \leq 2 + \frac{4a}{a_0} + 2\log (1+ \sum_{i=1}^{n-1} \frac{a_i}{a_0})
\end{equation}
\end{lem}

\begin{lem}~\citep{bach2019universal}
\label{app:lem:mrtg_diff}
Let $\sk \in \mathbb R$ be a convex set and $R: \sk \to \mathbb R$ be a $1$-strongly convex w.r.t.\ $\|.\|$. Also assume for all $x \in \sk$ we have $R(x) -\displaystyle\min_{y \in \sk} R(y) \leq 1/2 D^2$. Then for any martingale difference sequence $(Z_i)_{i=1}^n \in \mathbb R^d$ and any random vector $X$ defined on $\sk$ we have 
\begin{equation}
    \mathbb E \left[ \dpr{\sum_{i=1}^n Z_i}{X}\right] \leq D/2 \sqrt{\sum_{i=1}^n \mathbb E \|Z_i\|_*^2},
\end{equation}
where $\|.\|_*$ is the dual norm of $\|.\|$. 
\end{lem}

\vfill